\documentclass[12pt]{article}
\PassOptionsToPackage{numbers,compress}{natbib}
\usepackage[preprint]{neurips_2023}

\usepackage{amsmath,amssymb,amsthm}
\usepackage[
    bookmarks=true, 
    bookmarksnumbered=true, 
    bookmarkstype=toc, 
    colorlinks=true, 
    linkcolor=black, 
    citecolor=blue, 
    urlcolor=blue]{hyperref}
\usepackage[capitalize,noabbrev]{cleveref}
\usepackage{natbib}
\bibpunct[:]{(}{)}{,}{a}{}{,}
\usepackage{graphicx}
\usepackage{cleveref}
\usepackage{times}
\usepackage{comment}
\usepackage{nicefrac}
\usepackage{multirow}
\usepackage{tablefootnote}
\usepackage{color,colortbl}
\usepackage{bm}
\usepackage{physics}
\usepackage{mathtools}
\usepackage{algorithm}
\usepackage{algorithmic}
\usepackage{xcolor}
\usepackage{mysty}

\theoremstyle{plain}
\newtheorem{theorem}{Theorem}

\theoremstyle{definition}

\newtheorem{assumption}{Assumption}
\theoremstyle{remark}
\newtheorem{remark}{Remark}
\crefname{assumption}{assumption}{assumptions}
\Crefname{assumption}{Assumption}{Assumptions}
\newtheorem{lem}{Lemma}
\crefname{lem}{lemma}{lemmas}
\Crefname{lem}{Lemma}{Lemmas}

\crefname{table}{table}{tables}
\Crefname{table}{Table}{Tables}

\usepackage[whole]{bxcjkjatype}

\title{Multi-point Feedback of Bandit Convex Optimization with Hard Constraints}

\author{%
    Yasunari Hikima\\
    Artificial Intelligence Laboratory, Fujitsu Limited, Japan \\
    \texttt{hikima.yasunari@fujitsu.com}
}

\begin{document}

\maketitle

\begin{abstract}%
    This paper studies bandit convex optimization with constraints, where the learner aims to generate a sequence of decisions under partial information of loss functions such that the cumulative loss is reduced as well as the cumulative constraint violation is simultaneously reduced.
    We adopt the cumulative \textit{hard} constraint violation as the metric of constraint violation, which is defined by $\sum_{t=1}^{T} \max\{g_t(\bmx_t), 0\}$.
    Owing to the maximum operator, a strictly feasible solution cannot cancel out the effects of violated constraints compared to the conventional metric known as \textit{long-term} constraints violation.
    We present a penalty-based proximal gradient descent method that attains a sub-linear growth of both regret and cumulative hard constraint violation, in which the gradient is estimated with a two-point function evaluation.
    Precisely, our algorithm attains $O(d^2T^{\max\{c,1-c\}})$ regret bounds and $O(d^2T^{1-\frac{c}{2}})$ cumulative hard constraint violation bounds for convex loss functions and time-varying constraints, where $d$ is the dimensionality of the feasible region and $c\in[\frac{1}{2}, 1)$ is a user-determined parameter.
    We also extend the result for the case where the loss functions are strongly convex and show that both regret and constraint violation bounds can be further reduced.
\end{abstract}

\section{Introduction}
Bandit Convex Optimization (BCO) is a fundamental framework of sequential decision-making under uncertain environments and with limited feedback, which can be regarded as a structured repeated game between a learner and an environment \citep{hazan2016introduction,lattimore2020bandit}.
In this framework, a learner is given a convex feasible region $\cX\subseteq\Rd$ and the total number $T$ of rounds.
At each round, $t=1,2,\dots, T$, the learner makes decision $\bmx_t\in\cX$, and then a convex loss function $f_t:\cX\to\R$ is revealed. 
The learner cannot access the loss function $f_t$, but only the \textit{bandit} feedback is available, i.e., the learner can only observe the value of the loss at the point she committed to, i.e., $f_t(\bmx_t)$. 
The objective of the learner is to generate a sequence of decisions $\{\bmx_t\}_{t=1}^{T}\subseteq\cX$ that minimizes cumulative loss $\sum_{t=1}^{T} f_t(\bmx_t)$ under bandit feedback.
The performance of the learner is evaluated in terms of \textit{regret}, which is defined by
\begin{equation*}
    R_T \coloneqq \sum_{t=1}^T f_t(\bmx_t) - \min_{\bmx\in\cX}\sum_{t=1}^{T} f_t(\bmx).
\end{equation*}
This regret measures the difference between the cumulative loss of the learner's strategy and the minimum possible cumulative loss where the sequence of loss functions $\{f_t(\bmx)\}_{t=1}^{T}$ had been known in advance and the learner could choose the best fixed optimal decision in hindsight.

In many real-world scenarios, the decisions are often subject to some constraints such as budget or resources. 
In the context of Online Convex Optimization (OCO), where the learner has access to the complete information about the loss functions, a projection operator is typically applied in each round so that the decisions belong to constraints \citep{zinkevich2003online,hazan2016introduction}.
However, such a projection step is typically a computational bottleneck when the feasible region is complex.

To address the issue of the projection step, \citet{mahdavi2012trading} considers \textit{online convex optimization with long-term constraints}, where the learner aims to generate a sequence of decisions that the decisions satisfy constraints in the long run, instead of requiring to satisfy the constraints in all rounds. 
They introduce the cumulative \textit{soft} constraint violation metric defined by $V^{\text{soft}}_T\coloneqq\sum_{t=1}^{T}g_t(\bmx_t)$, where $g_t(\bmx)\leq 0$ is the functional constraint to be satisfied. Later, \citet{yuan2018online} consideres strict notion of constraint violation reffered to as cumulative \textit{hard} constraint violation, which is defined by $V^{\text{hard}}_T\coloneqq\sum_{t=1}^{T} \max\{g_t(\bmx_t), 0\}$. This metric overcomes the drawback of cumulative soft constraint violation, and it is suitable for safety-critical systems, in which the failure of constraint violation may result in catastrophic consequences.

To see that the notion of cumulative hard constraint violation is a stronger metric, let us consider the example discussed in \citet{guo2023rectified}. 
Given a sequence of decisions whose constraint functions are $\{g_t(\bmx_t)\}_{t=1}^{T}$ with $T=1000$ such that $g_t(\bmx_t) = -1$ if $t$ is odd; otherwise $g_t(\bmx_t) = 1$, we have $\sum_{t=1}^{\tau} g_t(\bmx_t) \leq 0$ for any $\tau\in\{1,2,\dots,T\}$, however, the constraint $g_t(\bmx)\leq 0$ is violated at half of rounds. 
On the other hand, the notion of hard constraint violation can capture the constraint violation since we have $V^{\text{hard}}_T = 500$.
Thus, the conventional definition of cumulative soft constraint violation $V_T^{\text{soft}}$ cannot accurately measure the constraint violation but cumulative hard constraint violation $V_T^{\text{hard}}$ can.

Many existing algorithms for BCO with constraints proposed in prior works typically involve projection operators as well as algorithms for OCO with constraints \citep{agarwal2010optimal,zhao2021bandit}, and are generally limited to the simple convex set. \citet{chen2019projection,garber2020improved} consider a projection-free algorithm for BCO, but the constraint violation bound has not been reported. Some studies have extended the algorithm for OCO with soft constraints to the bandit setting \citep{mahdavi2012trading,cao2018online}, however, these algorithms cannot be directly extended to BCO with hard constraints. In other words, there has been no algorithm that can simultaneously achieve sub-linear bound both regret and cumulative hard constraints violation.

The present study focuses on the particular case of multi-point feedback of BCO with constraints, in which the loss functions are convex or strongly convex, and constraint violation is evaluated in terms of hard constraints. 
This kind of problem widely appears in real-world scenarios such as portfolio management problems, in which the manager has concrete constraints to be satisfied but only has access to the loss function $f_t(\cdot)$ at several points close to the decision $\bmx_t$.
We present a penalty-based proximal gradient descent method which attains both $O(d^2T^{\max\{c,1-c\}})$ regret bound and $O(d^2T^{1-\frac{c}{2}})$ cumulative hard constraint violation bound, where $d$ is the dimensionality of the feasible region and $c\in[\frac{1}{2}, 1)$ is a user-determined parameter. 
Our proposed algorithm is inspired by a gradient estimation in the BCO literature \citep{flaxman2004online,agarwal2010optimal} and an algorithm for OCO with hard constraints \citep{guo2022online}.

\subsection{Related work}

For OCO with constraints, a projection operator is generally applied to the updated variables to enforce them feasible at each round \citep{zinkevich2003online,duchi2010composite}.
However, such projection is typically inefficient to implement due to the high computational effort especially when the feasible region $\cX$ is complex (e.g., $\cX$ is characterized by multiple inequalities), and efficient projection computation is limited to simple sets such as $\ell_1$-ball or probability simplex \citep{duchi2008efficient}.

Instead of requiring that the decisions belong to the feasible region in all rounds, \citet{mahdavi2012trading} first considers relaxing the notion of constraints by allowing them to be violated at some rounds but requiring them to be satisfied in the long run.
This type of OCO is referred to as \textit{online convex optimization with long-term constraints}, and the performance metric for constraint violation is defined by the cumulative violation of the decisions from the constraints for all rounds, i.e., $V^{\text{soft}}_T\coloneqq \sum_{t=1}^{T} g_t(\bmx_t)$ referred to as \textit{soft} constraints.
\citet{mahdavi2012trading} proposes a primal-dual gradient-based algorithm that attains $O(\sqrt{T})$ regret bound and $O(T^{\frac{3}{4}})$ constraint violations and subsequent researches have been conducted to improve both bounds.
\citet{jenatton2016adaptive} extends the algorithm to achieve $O(T^{\max\{c,1-c\}})$ regret bound and $O(T^{1-\frac{c}{2}})$ constraint violation, where $c\in(0,1)$ is a user-determined parameter.
\citet{yu2020low} proposes the drift-plus-penalty based algorithm developed for stochastic optimization in dynamic queue networks \citep{neely2022stochastic}, and prove the algorithm attains $O(\sqrt{T})$ regret bound and $O(1)$ constraint violation bound.

\citet{yuan2018online} proposes the more strict notion of a constraint violation, which is defined by $V^{\text{hard}}_T\coloneqq \sum_{t=1}^{T} \max\{g_t(\bmx_t),0\}$, so as not to cancel out the effect of violated constraints by the strict feasible solution.
Such paradigm is later referred to as \textit{online convex optimization with hard constraints} \citep{guo2022online}.
In \citet{yuan2018online}, an algorithm that attains $O(T^{\max\{c,1-c\}})$ regret bound and $O(T^{1-\frac{c}{2}})$ constraint violation bound has proposed.
\citet{yi2021regret} extends the algorithm that attains $O(T^{\max\{c,1-c\}})$ regret bound and $O(T^{\frac{1-c}{2}})$ constraint violation bound, and \citet{yi2021regret} also consideres the general dynamic regret bound.
\citet{guo2022online} proposes an algorithm that rectifies updated variables and penalty variables and proves the algorithm attains $O(\sqrt{T})$ regret bound and $O(T^{\frac{3}{4}})$ constraint violation for convex loss functions.

In the partial information setting, a learner is limited to accessing the loss functions and thus the learner cannot construct an algorithm by using a gradient of loss functions.
\citet{flaxman2004online} considers a one-point feedback model, where only one-point function value is available, and constructed an unbiased estimator of the gradient of the loss functions.
By employing the gradient estimator, they applied online gradient descent algorithm \citep{zinkevich2003online} and proved the algorithm attains $O(d^{\frac{2}{3}}T^\frac{2}{3})$ regret bound.
Another variant of the feedback model is multi-point feedback, where the learner is allowed to query multiple points of function in each round.
\citet{agarwal2010optimal} and \citet{nesterov2017random} consideres two-point feedback model and establishes an $O(d^2\sqrt{T})$ regret bound for convex loss functions.

\begin{table}[t]
\small
\centering
\caption{Regret bound and cumulative constraint violation bound for bandit convex optimization with constraints. The column of ``Metric" stands for the metric of constraint violation.}
\begin{tabular}{c|cccccc}\hline
Reference & Bandit & Metric & Loss & Regret  & Violation  \\ \hline\hline
\citet{flaxman2004online} & $\checkmark$ & --- & convex & $O(d^{\frac{2}{3}}T^{\frac{2}{3}})$ & --- \\ \hline
\multirow{2}{*}{\citet{agarwal2010optimal}} & $\checkmark$ & --- & convex & $O(d^2\sqrt{T})$ & --- \\  
& $\checkmark$ & --- & str.-convex & $O(d^2\log T)$ & ---  \\ \hline
\citet{mahdavi2012trading} & $\checkmark$ & soft & convex & $O(\sqrt{T})$ & $O(T^{\frac{3}{4}})$  \\ \hline
\multirow{2}{*}{\citet{guo2022online}} & & \multirow{2}{*}{hard} & convex & $O(\sqrt{T})$ & $O(T^{\frac{3}{4}})$  \\  
& & & str.-convex & $O(\log T)$ & $O(\sqrt{T(1+\log T)})$  \\ \hline
\cellcolor{black!20} & $\checkmark$ & & convex & \cellcolor{black!20}$O(d^2T^{\max\{c,1-c\}})$ & \cellcolor{black!20}$O(d^2T^{1-\frac{c}{2}})$ \\
\multirow{-2}{*}{\cellcolor{black!20}{This work}} & $\checkmark$ & \multirow{-2}{*}{hard} & str.-convex & \cellcolor{black!20}$O(d^2\log T)$ & \cellcolor{black!20}$O(d^2\sqrt{T(1+\log T)})$ \\ \hline
\end{tabular}
\label{table:comp}
\end{table}

\subsection{Contribution}
This paper focuses on the multi-point feedback BCO with constraints, in which the constraint violation is evaluated in terms of cumulative hard constraint violation.
We propose an algorithm (\Cref{alg:reRECOO}) for the BCO and show that the proposed algorithm attains an $O(d^2T^{\max\{c,1-c\}})$ regret bound and an $O(d^2T^{1-\frac{c}{2}})$ cumulative hard constraint violation bound, where $c\in[\frac{1}{2}, 1)$ is a user-defined parameter (\Cref{thm:main} and \Cref{thm:adversarial_constraint}).
By setting $c=\frac{1}{2}$, the algorithm attains $O(d^2\sqrt{T})$ regret bound and $O(d^2T^{\frac{3}{4}})$ constraint violation bound, which is compatible with the prior work for constrained online convex optimization with full-information \citep{yi2022regret,guo2022online}.
We also show both regret and constraint violation bounds are reduced to an $O(d^2\log T)$ and $O(d^2\sqrt{T(1+\log T)})$, respectively, when the loss functions are strongly convex (\Cref{thm:strongly_convex} and \Cref{thm:strongly_constraints}).
The comparison of this study with prior works is summarized in \Cref{table:comp}.

\subsection{Organization}
The rest of this paper is organized as follows.
In \Cref{sec:problem_formulation}, we introduce necessary preliminaries of BCO with constraints.
\Cref{sec:algorithm} presents the proposed algorithm to solve the BCO with constraints under two-point bandit feedback.
In \Cref{sec:theoretical_analysis}, we provide a theoretical analysis of regret bound and hard constraint violation bound for both convex and strongly convex loss functions.
Finally, \Cref{sec:conclusion} concludes the present paper and addresses future work.

\section{Preliminaries}\label{sec:problem_formulation}

\subsection{Notation}
For a vector $\bmx=(x_1,x_2,\dots,x_d)^\top\in\Rd$, let $\norm{\bmx}_2$ be the $\ell_2$-norm of $\bmx$, i.e., $\norm{x}_2=\sqrt{\bmx^\top\bmx}=\sqrt{\sum_{i=1}^{d}x_i^2}$.
Let $\left<\bmx,\bmy\right>$ be the inner product of two vectors $\bmx$ and $\bmy$.
Let $\mathbb{B}^d$ and $\mathbb{S}^d$ denote the $d$-dimensional Euclidean ball and unit sphere, and let $\bmv\in\mathbb{B}^d$ and $\bmu\in\mathbb{S}^d$ denote the random variables sampled uniformly from $\mathbb{B}^d$ and $\mathbb{S}^d$, respectively.
For a scalar $z\in\R$, we denote $[z]_+\coloneqq\max\{z,0\}$.
For a Lipschitz continuous function $f:\Rd\to\R$, let $\lip(f)>0$ be the Lipschitz constant of $f$.
We use $[T]$ as a shorthand for the set of positive integers $\{1,2,\dots,T\}$.
Finally, we use the notation $\mathbb{E}_t$ as the conditional expectation over the condition of all randomness in the first $t-1$ rounds.

\subsection{Assumptions}
Following prior works of constrained OCO \citep{mahdavi2012trading,guo2022online}, we make the following standard assumptions on feasible region, loss functions, and constraint functions.
\begin{assumption}[Bounded domain]\label{ass:X}
The feasible region $\cX\subseteq\R^d$ is a non-empty bounded closed convex set such that $\norm{\bmx - \bmy}_2 \leq D$ holds for any $\bmx,\,\bmy\in\cX$.
\end{assumption}
\begin{assumption}[Convexity and Lipschitz continuity of loss functions]\label{ass:ft}
    The loss function $f_t:\mathcal{X}\to\mathbb{R}$ is convex and Lipschitz continuous with Lipschitz constant $F_t>0$ on $\cX$, that is, we have
    \begin{align*}
        \abs{f_t(\bmx) - f_t(\bmy)} \leq F_t\norm{\bmx - \bmy}_2, 
    \end{align*}
    for any $\bmx,\bmy\in\cX$ and for any $t\in[T]$.
    For simplicity, we define $F:=\max_{t\in[T]}F_t$.
\end{assumption}
\begin{assumption}[Convexity and Lipschitz continuity of constraint functions]\label{ass:gt}
    The constraint function $g_t:\mathcal{X}\to\mathbb{R}$ is convex and Lipschitz continuous with Lipschitz constant $G_t>0$ on $\cX$, that is, we have
    \begin{align*}
        \abs{g_t(\bmx) - g_t(\bmy)} \leq G_t\norm{\bmx - \bmy}_2,
    \end{align*}
    for any  $\bmx,\bmy\in\cX$ and for any $t\in[T]$.
    For simplicity, we define $G:=\max_{t\in[T]}G_t$. 
\end{assumption}
\subsection{Offline constrained OCO}
With the full knowledge of loss functions $\{f_t(\bmx)\}_{t=1}^{T}$ and constraint functions $\{g_t(\bmx)\}_{t=1}^{T}$ in all rounds, the offline constrained OCO is formulated as the following convex optimization problem:
\begin{subequations}\label{opt:COCO}
    \begin{align}
        \min_{\bmx\in\cX}\quad&\sum_{t=1}^{T} f_t(\bmx) \\
        \text{subject to}\quad& g_t(\bmx) \leq 0 \qquad \forall t\in[T],
    \end{align}
\end{subequations}
where $\cX$ is assumed to be a simple convex set (e.g., Euclidean ball, probability simplex) for which the projection onto $\cX$ is efficiently computable.

For the sake of simplicity of theoretical analysis, the present paper considers the case where there exists a single constraint function.
By defining $g_t(\bmx) \coloneqq \max_{i\in[m]} g^{(i)}_t(\bmx)$, this study can be easily extended to the case where multiple constraint functions, i.e., $g_t^{(i)}(\bmx)\leq 0 \, (i\in[m])$ exist, because maximum of finite convex functions is also convex.

\subsection{Performance metrics}
Given a sequence of decisions $\{\bmx_t\}_{t=1}^{T} \subseteq \cX$ generated by some OCO algorithm (e.g., Online Gradient Descent method \citep{zinkevich2003online}).
Under the situation where all loss functions $\{f_t(\bmx)\}_{t=1}^{T}$ and constraint functions $\{g_t(\bmx)\}_{t=1}^{T}$ in each round $t=1,2,\dots,T$ are known in advance, the \textit{regret} and \textit{cumulative hard constraint violation} are defined as follows:
\begin{align}
    R_T &\coloneqq \sum_{t=1}^{T} f_t(\bmx_t) - \sum_{t=1}^{T} f_t(\bmx^\star), \label{eq:def_regret}\\
    V_T &\coloneqq \sum_{t=1}^{T} \qty[g_t(\bmx_t)]_+ = \sum_{t=1}^{T} \max\qty{g_t(\bmx_t), 0}, \label{eq:def_violation}
\end{align}
where $\bmx^\star\in\cX$ is the optimal solution to the offline constrained OCO formulated as Eq.~\eqref{opt:COCO}.
The objective of the learner is to generate a sequence of decisions that attains a sub-linear growth of both regret and cumulative constraint violation, that is, $\lim\sup_{T\to\infty} \frac{R_T}{T} \leq 0$ and $\lim\sup_{T\to\infty} \frac{V_T}{T} \leq 0$.

\subsection{Gradient estimator}

In the partial information setting where only limited feedback is available to the learner, we follow the prior works \citep{flaxman2004online,agarwal2010optimal,zhao2021bandit}.
The following result guarantees the gradient estimator with one-point feedback being an unbiased estimator.

\begin{lem}{ \cite[Lemma 1]{zhao2021bandit} }\label{lemma:gradient_estimate} 
    For any convex function $f:\cX\to\R$, define its smoothed version function $\widehat{f}(\bmx) = \mathbb{E}_{\bmv\in\mathbb{B}^d}[f(\bmx + \delta\bmv)]$, where the expectation is taken over the random vector $\bmv\in\mathbb{B}^d$ with $\B^d$ being the unit ball, i.e., $\B^d \coloneqq \qty{\bmx\in\Rd \mid \norm{\bmx}_2 \leq 1}$.
    Then, for any $\delta > 0$, we have
    \begin{align*}
        \E_{\bmu\in\mathbb{S}^d} \qty[\frac{d}{\delta} f(\bmx + \delta\bmu) \bmu] = \nabla \widehat{f}(\bmx),
    \end{align*}
    where the expectation is taken over the random vector $\bms\in\mathbb{S}^d$ with $\mathbb{S}^d$ being the unit sphere centered around the origin, i.e., $\mathbb{S}^d \coloneqq \qty{\bmx\in\Rd\mid\norm{\bmx}_2=1}$. 
\end{lem}
\begin{proof}
    See \citet[Lemma~2.1]{flaxman2004online}.
\end{proof}

Moreover, as shown in \citet[Lemma~8]{shamir2017optimal}, for any convex function $f:\cX\to\R$ and its smoothed version $\widehat{f}$, we have
\begin{align}
    \sup_{\bmx\in\cX} \abs{\widehat{f}(\bmx) - f(\bmx)} \leq \delta \lip(f). \label{eq:sup_f}
\end{align}

The present study considers a two-point feedback model where the learner is allowed to query two points in each round.
Specifically, at round $t\in[T]$, the learner is allowed to query two points around decision $\bmx_t$, that is, $\bmx_t + \delta\bmu_t$ and $\bmx_t - \delta\bmu_t$, where $\delta > 0$ is a perturbation parameter and $\bmu_t$ is a random unit vector sampled from unit sphere $\mathbb{S}^d$.
With two points $\bmx_t + \delta\bmu_t$ and $\bmx_t - \delta\bmu_t$, the gradient estimator of the function $f_t$ at $\bmx_t$ is given by
\begin{align}\label{eq:def_g_tilde}
    \widetilde{\nabla}f_t \coloneqq \frac{d}{2\delta} \qty[ f_t(\bmx_t + \delta \bmu_t) - f_t(\bmx_t - \delta \bmu_t)] \bmu_t,
\end{align}
where $d$ is the dimensionality of the domain $\cX\subseteq\Rd$.
As shown in \citet{agarwal2010optimal}, $\widetilde{\nabla}f_t$ is norm bounded, that is, we have $\|\widetilde{\nabla}f_t\|_2 \leq \frac{\delta}{2\delta} \lip (f_t) \norm{2\delta\bmu_t}_2 \leq \lip(f_t)d$, where the first inequality holds by the Lipschitz continuity of $f_t$.

\Cref{lemma:gradient_estimate} implies that the gradient estimator $\widetilde{\nabla} f_t$ is an unbiased estimator of $\nabla \widehat{f}_t(\bmx_t)$, i.e., $\mathbb{E}_{\bmu\in\mathbb{S}^d}[\widetilde{\nabla} f_t] = \nabla\widehat{f}_t(\bmx_t)$, where $\widehat{f}_t(\bmx_t) = \mathbb{E}_{\bmv\in\mathbb{B}^d}[f_t(\bmx_t+\delta\bmv)]$ is the smoothed version of original function $f_t$.
This property holds because the distribution of perturbation $\bmu_t$ in Eq.~\eqref{eq:def_g_tilde} is symmetric.

\section{Proposed Algorithm}\label{sec:algorithm}
This section presents the proposed algorithm for solving the constrained BCO with two-point feedback.
The procedure of the algorithm is shown in \Cref{alg:reRECOO}, and this algorithm is motivated by the work in \citet{guo2022online} and the design of the algorithm is related to penalty-based proximal gradient descent method \citep{cheung2017proximal}.
At round $t\in[T]$, \Cref{alg:reRECOO} finds the decision vector $\bmx_{t+1}$ by solving the following strongly convex optimization problem:
\begin{align}\label{eq:opt_decision_vector}
    \bmx_{t+1} = \arg\min_{\bmx\in(1-\xi)\cX} \qty{f_{t}(\bmx_{t}) + \widetilde{\nabla}f_{t}^\top (\bmx-\bmx_{t}) + \lambda_{t}\widehat{g}^+_{t}(\bmx) + \frac{\alpha_{t}}{2}\norm{\bmx - \bmx_{t}}^2_2},
\end{align}
where $\lambda_t$ is the penalty variable for controlling the quality of the decision, $\widehat{g}^+_t(\bmx) \coloneqq \gamma_t [g_t(\bmx)]_+$, $\xi>0$ is the shrinkage constant, and $\alpha_t>0,\,\gamma_t>0$ are predetermined learning rate.
Note that the optimization problem in the right-hand side (r.h.s) of Eq.~\eqref{eq:opt_decision_vector} is strongly convex optimization due to the $\ell_2$ regularizer term, and hence the optimal solution $\bmx_{t+1}$ does exist and unique.
As is the case with \citet{mahdavi2012trading}, we optimize the r.h.s. of Eq.~\eqref{eq:opt_decision_vector} on the domain $(1-\xi)\cX$ to ensure that randomized two points around $\bmx_t$ are inside the feasible region $\cX$.
As shown in \citet{flaxman2004online}, for any $\bmx\in(1-\xi)\cX$ and for any unit vector $\bmu\in\mathbb{S}^d$, it holds $\bmx \pm \delta\bmu\in\cX$.

At round $t$, where we find the decision $\bmx_{t+1}\in\cX$, since we do not have the prior knowledge of the loss function $f_{t+1}(\bmx)$ to be minimized, we estimate the loss by the first-order approximation at the previous decision $\bmx_{t}$ as $\widetilde{f}_{t+1}(\bmx) = f_{t}(\bmx_t) + \left<\nabla f_{t}(\bmx_t), \bmx-\bmx_{t}\right>$.
Simultaneously, we have no full information of the loss function $f_{t}(\bmx)$ and hence we cannot access its graient $\nabla f_t(\bmx)$, so we estimate gradient by $\widetilde{\nabla}f_t$ with two points (line~\ref{line_gradient}).
To prevent the constraint from being severely violated, we also introduce the \textit{rectified} Lagrange multiplier $\lambda_t$ associated with the functional constraint $g_t(\bmx)\leq 0$, and add the penalty term $\lambda_t\widehat{g}_t^+(\bmx)$ to the objective function \eqref{eq:opt_decision_vector}, which is an approximator of the original penalty term $\theta_tg_t(\bmx)$, where $\theta_t$ is the Lagrangian multiplier associated with the constraint $g_t(\bmx) \leq 0$.
We also add $\ell_2$ regularization term $\frac{\alpha_t}{2} \norm{\bmx - \bmx_t}^2_2$ to stabilize the optimization problem.

We will describe more in detail the role of penalty parameter $\lambda_t$ and its update rule.
The penalty parameter $\lambda_t$ is related to the Lagrangian multiplier (denoted by $\theta_t$) associated with the functional constraint $g_t(\bmx)\leq 0$, but slightly different because we have no prior knowledge of the constraint functions when making-decision.
Instead, we take place the original Lagrangian multiplier $\theta_{t+1}$ with $\lambda_t$ such that $\lambda_t\widehat{g}^+_t(\bmx)$ is an approximator of $\theta_t g_t(\bmx)$.
We update the penalty parameter (line~\ref{line_penalty}) as $\lambda_{t+1} = \max\{\lambda_{t} + \gamma_{t+1}[g_{t+1}(\bmx_{t})]_+, \eta_{t+1}\}$, where the first coordinate of maximum operator is the sum of the old $\lambda_t$ and the rectified constraint function value $\gamma_{t+1}[g_{t+1}(\bmx_{t})]_+$; and the second coordinate is the user-determined constant $\eta_{t+1}$ to impose a minimum penalty.
This update rule for the penalty parameter prevents the decision determined by solving the problem \eqref{eq:opt_decision_vector} from being overly aggressive which leads to large constraint violation.

\begin{algorithm}[t]
    \caption{A Rectified Bandit Convex Optimization with Hard Constraints under Two-Point Bandit Feedback}
    \label{alg:reRECOO}
    \begin{algorithmic}[1]
        \REQUIRE Total number of rounds $T$, learning rates $\{\alpha_t\}_{t=1}^{T}\subseteq\R_{>0},\, \{\gamma_t\}_{t=1}^{T}\subseteq\R_{>0},\, \{\eta_t\}_{t=1}^{T}\subseteq\R_{>0}$, shrinkage paramaeter $\xi>0$, and perturbation parameter $\delta>0$.
        \STATE{\textbf{Initialization}: $\bmx_1\in\cX,\, \lambda_1=0$, and set $\widehat{g}^+_{1}(\bmx) = \gamma_{1}[g_{1}(\bmx)]_+$.}
        \FOR{$t = 1,2,\dots,T$}
        \STATE Draw unit vector $\bmu_t$ from $\mathbb{S}^d$ uniformly at random.
        \STATE Query $f_t(\bmx)$ at two points $\bmx_t+\delta\bmu_t$ and $\bmx_t-\delta\bmu_t$.
        \STATE Compute $\widetilde{\nabla}f_{t} = \frac{d}{2\delta} \qty[f_t(\bmx_t+\delta\bmu_t) - f_t(\bmx_t-\delta\bmu_t)]\bmu_t$. \label{line_gradient}
        \STATE Find the optimal solution $\bmx_{t+1}$ by solving optimization problem \eqref{eq:opt_decision_vector}.
        \STATE Submit $\bmx_{t+1}$, incur loss $f_{t+1}(\bmx_{t+1})$ and observe constraint $g_{t+1}(\bmx)$.
        \STATE Set $\widehat{g}^+_{t+1}(\bmx) = \gamma_{t+1}[g_{t+1}(\bmx)]_+$. \label{line_rectify}
        \STATE Update the penalty variable as $\lambda_{t+1} = \max\{\lambda_{t} + \gamma_{t+1}[g_{t+1}(\bmx_{t})]_+, \eta_{t+1}\}$. \label{line_penalty}
        \ENDFOR
    \end{algorithmic}
\end{algorithm}

\section{Theoretical Analysis}\label{sec:theoretical_analysis}

This section provides the theoretical analysis for the \Cref{alg:reRECOO}.
To facilitate the analysis, let $h_t:\cX\to\R$ be a function defined by
\begin{equation}\label{eq:def_h_t}
    h_t(\bmx) \coloneqq \widehat{f}_t(\bmx) + \left<\widetilde{\nabla}f_t - \nabla \widehat{f}_t(\bmx_t), \bmx \right>,
\end{equation}
where $\widehat{f}_t(\bmx) = \mathbb{E}_{\bmv\in\mathbb{B}^d}[f_t(\bmx + \delta\bmv)]$ and $\widetilde{\nabla}f_t$ is defined as Eq.~\eqref{eq:def_g_tilde}.
It is easily seen that $\nabla h_t(\bmx_t) = \widetilde{\nabla}f_t$ holds, and hence we have $\norm{\nabla h_t(\bmx)}_2 = \|\widehat{\nabla}f_t\|_2 \leq d\text{lip}(f_t)$ for any $\bmx\in\cX$.
Moreover, the function $h_t$ defined as Eq.~\eqref{eq:def_h_t} is convex and Lipschitz continuous with Lipschitz constant $\lip(h_t)=3d\lip(f_t)$ on $\cX$, because for any $\bmx,\bmy\in\cX$, we have
\begin{align*}
    \abs{h(\bmx) - h(\bmy)} 
    &\leq \abs{\widehat{f}_t(\bmx) - \widehat{f}_t(\bmy)} + \abs{\left<\widetilde{\nabla}f_t - \nabla \widehat{f}_t(\bmx_t), \bmx - \bmy\right>} \\
    &\leq \lip(\widehat{f}_t) \norm{\bmx - \bmy}_2 + \qty(\|\widetilde{\nabla}f_t\|_2 + \|\nabla \widehat{f}_t(\bmx_t)\|_2) \norm{\bmx - \bmy}_2 \\
    &\leq \lip(\widehat{f}_t)\norm{\bmx-\bmy}_2 + \qty(\lip(\widehat{f}_t)d + \lip(\widehat{f}_t))\norm{\bmx - \bmy}_2 \leq 3d\lip(f_t)\norm{\bmx - \bmy}_2,
\end{align*}
where the first inequality follows from the triangle inequality, the second inequality follows from the Cauchy-Schwarz inequality, the third inequality follows from $\norm{\nabla f (\bmx)}_2\leq \lip(f)$ for any Lipshitz continuous function $f$ and for any $\bmx\in\cX$, and the last inequality follows from $\lip(\widehat{f}_t) = \lip(f_t)$. 

To prove \Cref{alg:reRECOO} attains sub-linear bound for both regret and cumulative hard constraint violation, we first show the following result which is a well-known property of a strongly convex function.

\begin{lem}{ \cite[Theorem 2.1.8]{nesterov2018lectures} }\label{lemma:strong_convex_property}
    Let $\cX\subseteq\Rd$ be a convex set. Let $f:\cX\to\mathbb{R}$ be a strongly convex function with modulus $\sigma$ on $\cX$, and let $\bmx^\star\in\cX$ be an optimal solution of $f$, that is, $\bmx^\star=\arg\min_{\bmx\in\cX} f(\bmx)$.
    Then, $f(\bmx) \geq f(\bmx^\star) + \frac{\sigma}{2}\norm{\bmx - \bmx^\star}^2_2$ holds for any $\bmx\in\cX$.
\end{lem}

\begin{proof}
    By the definition of strong convexity of $f$, for any $\bmx,\,\bmy\in\cX$, we have
    \begin{align}\label{eq:strong_convexity}
        f(\bmx) \geq f(\bmy) + \left< \nabla f(\bmx), \bmx - \bmy \right> + \frac{\sigma}{2} \norm{\bmx - \bmy}^2_2.
    \end{align}
    Plugging an optimal solution $\bmx^\star\in\cX$ into $\bmy$ in the above inequality \eqref{eq:strong_convexity}, we have
    \begin{align*}
        f(\bmx) &\geq f(\bmx^\star) + \left< \nabla f(\bmx^\star), \bmx - \bmx^\star \right> + \frac{\sigma}{2} \norm{\bmx - \bmx^\star}^2_2 
        \geq f(\bmx^\star) + \frac{\sigma}{2} \norm{\bmx - \bmx^\star}^2_2,
    \end{align*}
    where the last inequality holds by the first-order optimality condition, $\left<\nabla f(\bmx^\star), \bmx - \bmx^\star\right> \geq 0$.
\end{proof}

The following two lemmas play an important role in proving the main theorem (\Cref{thm:main} and \Cref{thm:adversarial_constraint}).
The first one (\Cref{lemma:self_bounding_property}) is an inequality involving the update rule of \Cref{alg:reRECOO}, and the second one (\Cref{lemma:key_property}) characterizes the relationship between the current solution $\bmx_t$ in \Cref{alg:reRECOO} and the optimal solution of the offline optimization problem formulated as Eq.~\eqref{opt:COCO}.

\begin{lem}{ \cite[Lemma 5]{guo2022online} }\label{lemma:self_bounding_property}
    Let $\varphi_t:\cX\to\Rd$ be a function defined by
    \begin{align}
        \varphi_t(\bmx)\coloneqq f_t(\bmx_t) + \left<\nabla f_t(\bmx_t), \bmx - \bmx_t \right> + \lambda_t \widehat{g}^+_t(\bmx) + \frac{\alpha_t}{2}\norm{\bmx - \bmx_t}^2_2,
    \end{align}
    where $\widehat{g}^+_t(\bmx) \coloneqq\gamma_t g_t(\bmx)$ and $\alpha_t > 0, \gamma_t > 0$ are predetermined learning rate. Let $\bmx_{t+1}$ be the optimal solution returned by \Cref{alg:reRECOO} where the gradient $\nabla f_t(\bmx)$ is accessible, that is, $\bmx_{t+1} = \arg\min_{\bmx\in\cX} \varphi_t(\bmx)$.
    Then, for any $\bmx\in\cX$, we have
    \begin{align}\label{eq:lemma_5_guo}
    \begin{split}
        &f_t(\bmx_t) + \left<\nabla f_t(\bmx_t), \bmx_{t+1} - \bmx_{t}\right> + \lambda_t \widehat{g}^+_t(\bmx_{t+1}) + \frac{\alpha_t}{2} \norm{\bmx_{t+1} - \bmx_{t}}^2_2 \\
        &\quad\leq f_t(\bmx_t) + \left<\nabla f_t(\bmx_t), \bmx - \bmx_{t}\right> + \lambda_t \widehat{g}^+_t(\bmx) + \frac{\alpha_t}{2} \norm{\bmx - \bmx_{t}}^2_2 - \frac{\alpha_t}{2} \norm{\bmx - \bmx_{t+1}}^2_2.
    \end{split}
    \end{align}
\end{lem}
\begin{proof}
Since $\varphi_t$ is a strongly convex function with modulus $\alpha_t$, we can apply \Cref{lemma:strong_convex_property} to $\varphi_t$.
Thus, we have $\varphi_t(\bmx_{t+1}) \leq \varphi_t(\bmx) - \frac{\alpha_t}{2} \norm{\bmx - \bmx_{t+1}}^2_2$ for any $\bmx\in\cX$, which completes the proof.
\end{proof}

\begin{lem}[Self-bounding Property]{ \cite[Lemma 1]{guo2022online} }\label{lemma:key_property}
    Let $f_t:\cX\to\R$ be a convex function satisfying \Cref{ass:ft}. Let $\bmx^\star\in\cX$ be any optimal solution to the offline constrained OCO of Eq. \eqref{opt:COCO} and $\bmx_t\in\cX$ be the optimal solution returned by \Cref{alg:reRECOO}. Then, we have
    \begin{equation}
        f_t(\bmx_t) - f_t(\bmx^\star) + \lambda_t \widehat{g}^+_t(\bmx_{t+1})
        \leq \frac{F_t^2}{4\alpha_t} + \frac{\alpha_t}{2}\norm{\bmx^\star-\bmx_t}^2_2 - \frac{\alpha_t}{2}\norm{\bmx^\star-\bmx_{t+1}}^2_2,
    \end{equation}
    where $\widehat{g}^+_t(\bmx) \coloneqq\gamma_t g_t(\bmx)$ and $\alpha_t > 0, \gamma_t > 0$ are predetermined learning rate.
\end{lem}
\begin{proof}
    See \citet[Lemma~1]{guo2022online}.
\end{proof}

We are now ready to prove the main results, which state \Cref{alg:reRECOO} achieves a sub-linear bound for both regret \eqref{eq:def_regret} and cumulative hard constraint violation \eqref{eq:def_violation}.
We first show the case where the loss functions are convex and constraint functions are fixed throughout the whole round.

\subsection{Convex loss function case}

\begin{theorem}\label{thm:main}
    Let $\{\bmx_t\}_{t=1}^{T}$ be a sequence of decisions generated by \Cref{alg:reRECOO} and let $\bmx^\star\in\cX$ be an optimal solution to the offline OCO of Eq.~\eqref{opt:COCO}.
    Assume that constraint functions are fixed, that is, $g_t(\bmx) = g(\bmx)$ for any $t\in[T]$.
    Define $\alpha_t\coloneqq t^c,\, \gamma_t \coloneqq t^{c + \varepsilon},\, \eta_t \coloneqq t^c$ and $\delta \coloneqq \frac{1}{T}$, where $c\in[\frac{1}{2}, 1)$ and $\varepsilon > 0$.
    Under \Cref{ass:X,ass:ft,ass:gt}, we have
    \begin{align}
        \sum_{t=1}^{T} \qty[f_t(\bmx_t) - f_t(\bmx^\star)] &\leq \qty(\frac{9F^2d^2}{4(1-c)} + \frac{D^2}{2} + 2F) T^{\max\{c, 1-c\}} = O(d^2T^{\max\{c, 1-c\}}), \label{eq:regret_thm_1}\\
        \sum_{t=1}^{T} [g_{t} (\bmx_t)]_+ &\leq \frac{27F^2d^2}{4} + \frac{3FdD (1+\varepsilon)}{\varepsilon} + D^2 = O(d^2). \label{eq:violation_thm_1}
    \end{align}
\end{theorem}

\begin{proof}
Similar to the argumant in \citet{flaxman2004online} and \citet{agarwal2010optimal}, letting $\bmxi_t \coloneqq \widetilde{\nabla}f_t - \nabla \widehat{f}_t(\bmx_t)$, then we have $\mathbb{E}_{t}[\bmxi_t] = \bm{0}$ from \Cref{lemma:gradient_estimate}, and thus, we have $\mathbb{E}_{t}[\bmxi^\top\bmx] = 0$ for any fixed $\bmx\in\cX$.
Therefore, for any fixed $\bmx\in\cX$, we have
\begin{align*}
    \mathbb{E}_t[h_t(\bmx)] &= \ex[t]{\widehat{f}_t(\bmx)} + \ex[t]{\bmxi_t^\top\bmx} = \widehat{f}_t(\bmx).
\end{align*}
\paragraph{Part (i): Proof of Eq.~\eqref{eq:regret_thm_1}}
Recall that the function $h_t$ is Lipschitz continuous with Lipschitz constant $\lip(h_t)=3F_t d$.
Applying \Cref{lemma:key_property} to the convex function $h_t$ defined by Eq.~\eqref{eq:def_h_t}, for an optimal solution $\bmx^\star$ to the offline optimization problem as Eq.~\eqref{opt:COCO}, we have
\begin{align*}
     \sum_{t=1}^{T} \qty[h_t(\bmx_t) - h_t(\bmx^\star)] 
     &\leq \sum_{t=1}^{T} \frac{\lip (h_t)^2}{4\alpha_t} + \sum_{t=1}^{T} \qty(\frac{\alpha_t}{2}\norm{\bmx^\star - \bmx_t}^2_2 - \frac{\alpha_t}{2}\norm{\bmx^\star - \bmx_{t+1}}^2_2) \\
     &\leq \frac{9F^2 d^2}{4} \sum_{t=1}^{T} \frac{1}{\alpha_t} + \sum_{t=1}^{T} \qty(\frac{\alpha_t}{2} - \frac{\alpha_{t-1}}{2}) \norm{\bmx^\star - \bmx_{t}}^2_2 - \frac{\alpha_T}{2}\norm{\bmx^\star - \bmx_{T+1}}^2_2 \\
     &\leq \frac{9F^2 d^2}{4} \sum_{t=1}^{T} \frac{1}{\alpha_t} + D^2 \sum_{t=1}^{T} \qty(\frac{\alpha_t}{2} - \frac{\alpha_{t-1}}{2}),
\end{align*}
where the last inequality follows from \Cref{ass:X}.
Plugging in $\alpha_t = t^c$, we have
\begin{align*}
    \sum_{t=1}^{T} \qty[h_t(\bmx_t) - h_t(\bmx^\star)]
    \leq \frac{9F^2 d^2}{4} \cdot \frac{T^{1-c}}{1-c} + \frac{D^2}{2} T^c
    =\qty(\frac{9F^2 d^2}{4(1-c)} + \frac{D^2}{2}) T^{\max\{c,1-c\}}.
\end{align*}
Since we have $\ex[t]{h_t(\bmx)} = \widehat{f}(\bmx)$, by taking expectation, we have
\begin{align*}
    \sum_{t=1}^{T} \qty[\widehat{f}_t(\bmx_t) - \widehat{f}_t(\bmx^\star)] \leq \qty(\frac{9F^2 d^2}{4(1-c)} + \frac{D^2}{2}) T^{\max\{c,1-c\}}.
\end{align*}
From the inequality \eqref{eq:sup_f}, for any optimal solution $\bmx^\star\in\cX$ to the offline OCO as Eq.~\eqref{opt:COCO}, we have
\begin{align*}
    f_t(\bmx_t) - f_t(\bmx^\star) \leq \widehat{f}_t(\bmx_t) - \widehat{f}_t(\bmx^\star) + 2\delta F_t,
\end{align*}
for any $t\in[T]$.
Therefore, we have
\begin{align*}
    \sum_{t=1}^{T} \qty[f_t(\bmx_t) - f_t(\bmx^\star)] &\leq 
    \sum_{t=1}^{T} \qty[\widehat{f}_t(\bmx_t) - \widehat{f}_t(\bmx^\star)] + \sum_{t=1}^{T} 2\delta F_t \\
    &\leq \qty(\frac{9F^2 d^2}{4(1-c)} + \frac{D^2}{2})T^{\max\{c,1-c\}} + 2F  \\
    &\leq \qty(\frac{9F^2 d^2}{4(1-c)} + \frac{D^2}{2} + 2F)T^{\max\{c,1-c\}},
\end{align*}
where the second inequality follows by plugging in $\delta = \frac{1}{T}$.

\paragraph{Part (ii): Proof of Eq.~\eqref{eq:violation_thm_1}} 
From \Cref{lemma:key_property}, for any optimal solution $\bmx^\star\in\cX$ to the offline constrained OCO as Eq.~\eqref{opt:COCO}, we have
\begin{align*}
    \lambda_t\widehat{g}^+_t(\bmx_{t+1}) 
    \leq \frac{\lip(h_t)^2}{4\alpha_t} + \abs{h_t(\bmx_t) - h_t(\bmx^\star)} + \frac{\alpha_t}{2} \norm{\bmx^\star - \bmx_t}^2_2 - \frac{\alpha_t}{2} \norm{\bmx^\star - \bmx_{t+1}}^2_2.
\end{align*}
By the definition of $\widehat{g}^+_t(\bmx_{t+1})$, i.e., $\widehat{g}^+_{t+1}(\bmx) = \gamma_{t}[g_t(\bmx)]_+$, and plugging in $\alpha_t=\eta_t=t^c$, we have
\begin{align*}
    [g_{t}(\bmx_{t+1})]_+ &\leq \frac{9F_t^2 d^2}{4\lambda_t\alpha_t\gamma_t} + \frac{\abs{h_t(\bmx_t) - h_t(\bmx^\star)}}{\lambda_t\gamma_t} + \frac{\alpha_t}{2\lambda_t\gamma_t} \norm{\bmx^\star - \bmx_t}^2_2 - \frac{\alpha_t}{2\lambda_t\gamma_t} \norm{\bmx^\star - \bmx_{t+1}}^2_2 \\
    &\leq\frac{9F_t^2 d^2}{4t^{3c + \varepsilon}} + \frac{\abs{h_t(\bmx_t) - h_t(\bmx^\star)}}{t^{2c+\varepsilon}} + \frac{1}{t^{c + \varepsilon}} \qty(\norm{\bmx^\star - \bmx_{t}}^2_2 - \norm{\bmx^\star - \bmx_{t+1}}^2_2),
\end{align*}
where the second inequality is followed by $\lambda_t \geq \eta_t$, and we plugging $\alpha_t = \eta_t = t^c$ and $\gamma_t = t^{c + \varepsilon}$.
By taking summation over $t=1,2,\dots,T$, we have
\begin{align*}
    \sum_{t=1}^{T} [g_t(\bmx_{t+1})]_+
    &\leq \sum_{t=1}^{T}\frac{9F_t^2 d^2}{4t^{3c + \varepsilon}} + \sum_{t=1}^{T}\frac{\abs{h_t(\bmx_t) - h_t(\bmx^\star)}}{t^{2c+\varepsilon}} +\sum_{t=1}^{T}\frac{1}{t^{c + \varepsilon}} \qty(\norm{\bmx^\star - \bmx_{t}}^2_2 - \norm{\bmx^\star - \bmx_{t+1}}^2_2) \\
    &\leq \frac{27F^2 d^2}{4} + \frac{3FdD (1+\varepsilon)}{\varepsilon} + D^2,
\end{align*}
where the second inequality holds from \Cref{lemma:useful_inequality} in \Cref{appendix:useful}, which completes the proof.
\end{proof}

\begin{remark}
    By setting constant $c = \frac{1}{2}$, \Cref{alg:reRECOO} attains $O(d^2\sqrt{T})$ regret bound.
    This regret bound is compatible with the prior works of unconstrained bandit convex optimization \citep{agarwal2010optimal}, and is compatible with the result for full-information setting \citep{guo2022online}.
\end{remark}

For the case where the constraint functions are time-varying, we can show the following result.

\begin{theorem}\label{thm:adversarial_constraint}
    Let $\{\bmx_t\}_{t=1}^{T}$ be a sequence of decisions generated by \Cref{alg:reRECOO}.
    Assume that constraint functions $g_t(\bmx)$ are time-varying.
    Define $\alpha_t:=t^c,\, \gamma_t := t^{c + \varepsilon}$, and $\eta_t := t^c$, where $c\in[\frac{1}{2}, 1)$ and $\varepsilon > 0$.
    Under \Cref{ass:X,ass:ft,ass:gt}, we have
    \begin{align}
        \sum_{t=1}^{T} [g_{t} (\bmx_t)]_+ \leq \qty(\frac{27F^2d^2 + G^2}{4} + 3FdD \qty(8 + \frac{1}{\varepsilon}) + 2D^2) T^{1 - \frac{c}{2}} = O(d^2T^{1 - \frac{c}{2}}).
    \end{align}
\end{theorem}

\begin{proof}
    By the convexity of $[g_t(\bmx_t)]_+$ and \Cref{ass:gt}, we can show $[g_t(\bmx_t)]_+ - [g_t(\bmx_{t+1})]_+$ is upper bounded by $[g_t(\bmx_t)]_+ - [g_t(\bmx_{t+1})]_+ \leq \frac{G^2}{4\beta} + \beta\norm{\bmx_t - \bmx_{t+1}}^2_2$ for any $\beta > 0$ \citep[Lemma~2]{guo2022online}.
    Applying \Cref{lemma:key_property} to the function $h_t$ defined by Eq.~\eqref{eq:def_h_t}, for any $\bmx^\star\in\cX$, we have
    \begin{align*}
        \norm{\bmx_t - \bmx_{t+1}}^2_2 \leq \frac{2}{\alpha_t} \qty(h_t(\bmx^\star) - h_t(\bmx_t)) + \frac{2}{\alpha_t} \left<\nabla h_t(\bmx_t), \bmx_t - \bmx_{t+1} \right> + \norm{\bmx^\star - \bmx_t}^2_2 - \norm{\bmx^\star - \bmx_{t+1}}^2_2.
    \end{align*}
    By taking summation over $t=1,2,\dots,T$, we have
    \begin{align*}
        &\sum_{t=1}^{T} \norm{\bmx_t - \bmx_{t+1}}^2_2 \\
        &\quad\leq \sum_{t=1}^{T} \frac{h_t(\bmx^\star) - h_t(\bmx_t)}{\frac{1}{2}\alpha_t} + \sum_{t=1}^{T} \frac{\left<\nabla h_t(\bmx_t), \bmx_t - \bmx_{t+1} \right>}{\frac{1}{2}\alpha_t} + \sum_{t=1}^{T} \qty(\norm{\bmx^\star - \bmx_t}^2_2 - \norm{\bmx^\star - \bmx_{t+1}}^2_2) \\
        &\quad\leq \sum_{t=1}^{T} \frac{2\lip(h_t)D}{\frac{1}{2}\alpha_t} + \norm{\bmx^\star - \bmx_1}^2_2 \leq \frac{12FdD}{1-c}T^{1-c} + D^2,
    \end{align*}
    where the last inequality holds by plugging in $\alpha_t = t^c$.
    Therefore, we have 
    \begin{align*}
        \sum_{t=1}^{T}[g_t(\bmx_t)]_+
        &\leq \sum_{t=1}^{T} [g_t(\bmx_{t+1})]_+ + \frac{G^2 T}{4\beta} + \beta\sum_{t=1}^{T}\norm{\bmx_{t} - \bmx_{t+1}}^2_2 \\
        &\leq \frac{27F^2 d^2}{4} + \frac{3FdD (1+\varepsilon)}{\varepsilon} + D^2 + \frac{G^2 T}{4\beta} + \beta\qty(\frac{12FdD}{1-c}T^{1-c} + D^2) \\
        &\leq \frac{27F^2 d^2}{4} + \frac{3FdD (1+\varepsilon)}{\varepsilon} + D^2  + \qty(\frac{G^2}{4} + 24FdD + D^2)T^{1-\frac{c}{2}},
    \end{align*}
    where the second inequality follows from Eq.~\eqref{eq:violation_thm_1} in \Cref{thm:main} and the last inequality holds by plugging in $\beta = T^{\frac{c}{2}}$, which completes the proof.
\end{proof}

\begin{remark}
    By setting constant $c=\frac{1}{2}$, we can obtain $O(d^2T^{\frac{3}{4}})$ constraint violation bound.
    This bound is compatible with the result for full-information case \citep{guo2022online}.
\end{remark}

\subsection{Strongly convex loss function case}\label{sec:strong_convex_case}
We extend the results discussed in the previous subsection to the case where the loss functions are strongly convex.
We omit the proofs of the following results here since the technique of the proof is similar to that of \Cref{thm:main} and \Cref{thm:adversarial_constraint}.
These proofs are found in \Cref{appendix:proof_regret} and \Cref{appendix:proof_constr}.
To discuss the strongly convex case, we make the following assumption about loss functions.
\begin{assumption}[Strong convexity of loss functions]\label{ass:ft_strong}
    The loss function $f_t:\cX\to\R$ is Lipschitz continuous with Lipschitz constant $F_t$, and strongly convex on $\cX$ with modulus $\sigma_t>0$, i.e., we have 
    \begin{align}
        f_t(\bmy) \geq f_t(\bmx) + \left<\nabla f_t(\bmx), \bmy - \bmx\right> + \frac{\sigma_t}{2} \norm{\bmy - \bmx}^2_2,
    \end{align}
    for any $\bmx,\bmy\in\cX$ and for any $t\in[T]$.
    For simplicity, we define $\sigma \coloneqq \max_{t\in[T]} \sigma_t$.
\end{assumption}

Under \Cref{ass:ft_strong}, the function $h_t:\cX\to\R$ defined as Eq.~\eqref{eq:def_h_t} is also strongly convex with modulus $\sigma_t$, namely, $h_t(\bmy) \geq h_t(\bmx) + \left<\nabla h_t(\bmx), \bmy - \bmx\right> + \frac{\sigma_t}{2} \norm{\bmy - \bmx}^2_2$ for any $\bmx,\bmy\in\cX$.
Then, we can show the following results.

\begin{theorem}\label{thm:strongly_convex}
    Let $\{\bmx_t\}_{t=1}^{T}$ be a sequence of decisions generated by \Cref{alg:reRECOO} and let $\bmx^\star\in\cX$ be an optimal solution to the offline OCO of Eq.~\eqref{opt:COCO}.
    Assume that constraint functions are fixed, that is, $g_t(\bmx) = g(\bmx)$ for any $t\in[T]$.
    Define $\alpha_t\coloneqq\sigma t,\,\gamma_t\coloneqq t^{c+\varepsilon},\,\eta_t\coloneqq t^c$, and $\delta\coloneqq\frac{1}{\delta}$, where $c\in[\frac{1}{2}, 1)$ and $\varepsilon>0$.
    Under \Cref{ass:X,ass:gt,ass:ft_strong}, we have
    \begin{align*}
        \sum_{t=1}^{T} \qty[f_t(\bmx_t) - f_t(\bmx^\star)] &\leq \qty(\frac{9F^2d^2}{4\sigma} + 2F) \qty(1 + \log T) = O(d^2\log T), \\
        \sum_{t=1}^{T} [g_t(\bmx_t)]_+ &\leq \frac{27F^2d^2}{4\sigma} + \frac{3FdD(1+\varepsilon)}{\varepsilon} = O(d^2).
    \end{align*}
\end{theorem}

\begin{theorem}\label{thm:strongly_constraints}
    Let $\{\bmx_t\}_{t=1}^{T}$ be a sequence of decisions generated by \Cref{alg:reRECOO}. 
    Assume that constraint functions $g_t(\bmx)$ are time-varying.
    Define $\alpha_t\coloneqq\sigma t,\,\gamma_t\coloneqq t^{c+\varepsilon},\,\eta_t\coloneqq t^c$, where $c\in[\frac{1}{2}, 1)$ and $\varepsilon>0$.
    Under \Cref{ass:X,ass:gt,ass:ft_strong}, we have
    \begin{align*}
        \sum_{t=1}^{T} [g_t(\bmx_t)]_+ \leq
        \qty(\frac{27F^2d^2}{4\sigma} + \frac{G^2}{4} + 3FdD \qty(1 + \frac{1}{\varepsilon} + \frac{4}{\sigma}) + D^2) \sqrt{T(1 + \log T)}.
    \end{align*}
\end{theorem}

\section{Conclusion and Future Directions}\label{sec:conclusion}
This paper studies the two-point feedback of bandit convex optimization with constraints, in which the loss functions are convex or strongly convex, constraint functions are fixed or time-varying, and the constraint violation is evaluated in terms of cumulative hard constraint violation \citep{yuan2018online}.
We present a penalty-based proximal gradient descent algorithm with an unbiased gradient estimator and show that the algorithm attains a sub-linear growth of both regret and cumulative hard constraint violation.
It would be of interest to extend this work to the case where both the loss functions and constraint functions are bandit setup as discussed in \citet{cao2018online}, and the case where only one-point bandit feedback is available to the learner.
Furthermore, theoretical analysis of dynamic regret, where the comparator sequence can be chosen arbitrarily from the feasible
set, would be an important direction for future work.

\ack{The author would like to thank Dr. Sho Takemori for making a number of valuable suggestions and advice.}

\bibliographystyle{unsrtnat}


\newpage
\appendix

\section{Proof of useful inequalities}\label{appendix:useful}

To show \Cref{thm:main}, we present the following results which is similar argument of \citet[Lemma~6]{guo2022online}.

\begin{lem}\label{lemma:useful_inequality}
    Let $\bmx^\star\in\cX$ be an optimal solution to the offline constrained OCO defined as Eq. \eqref{opt:COCO}.
    Under \Cref{ass:X,ass:ft}, for any feasible solution $\bmx\in\cX$, $c\in[\frac{1}{2}, 1)$, and  $\varepsilon > 0$, we have
    \begin{align*}
        &\sum_{t=1}^{T} \frac{1}{t^{3c + \varepsilon}} \leq \frac{3c+\varepsilon}{3c+\varepsilon-1} \leq 3, \\
        &\sum_{t=1}^{T} \frac{\abs{f_t(\bmx_t) - f_t(\bmx^\star)}}{t^{2c+\varepsilon}} \leq \frac{FD(2c + \varepsilon)}{2c+\varepsilon-1}, \\
        &\sum_{t=1}^{T} \frac{\norm{\bmx_t - \bmx^\star}^2_2 - \norm{\bmx_{t+1} - \bmx^\star}^2_2}{t^{c+\varepsilon}} \leq D^2.
    \end{align*}
\end{lem}

\begin{proof}
    The first claim is shown as follows:
    \begin{align*}
        \sum_{t=1}^{T} \frac{1}{t^{3c  + \varepsilon}} \leq 1 + \int_{1}^{T} \frac{1}{t^{3c + \varepsilon}} \, \dd{t} = 1 + \frac{1 - T^{-3c-\varepsilon+1}}{3c+\varepsilon-1} \leq \frac{3c+\varepsilon}{3c + \varepsilon - 1} \leq 3,
    \end{align*}
    where the last inequality holds from the condition of $\frac{1}{2}\leq c < 1$.
    
    The second claim is shown as follows:
    \begin{align*}
        \sum_{t=1}^{T} \frac{\abs{f_t(\bmx_t) - f_t(\bmx^\star)}}{t^{2c+\varepsilon}} &\leq
        \sum_{t=1}^{T} \frac{F_t\norm{\bmx_t - \bmx^\star}_2}{t^{2c+\varepsilon}} 
        \leq \sum_{t=1}^{T} \frac{F D}{t^{2c+\varepsilon}}
        \leq \frac{FD(2c + \varepsilon)}{2c+\varepsilon-1},
    \end{align*}
    where the first inequality follows from \Cref{ass:ft}, the second inequality follows from \Cref{ass:X}, and the third inequality follows as is the case with the inequality $\sum_{t=1}^{T} \frac{1}{t^{3c + \varepsilon}} \leq \frac{3c+\varepsilon}{3c+\varepsilon-1}$.
    
    The last claim is shown as follows:
    \begin{align*}
        &\sum_{t=1}^{T} \frac{\norm{\bmx_t - \bmx^\star}^2_2 - \norm{\bmx_{t+1} - \bmx^\star}^2_2}{t^{c+\varepsilon}} \\
        &\quad= \norm{\bmx_1 - \bmx^\star}^2_2 + \sum_{t=2}^{T} \qty(\frac{1}{t^{c + \varepsilon}} - \frac{1}{(t - 1)^{c + \varepsilon}}) \norm{\bmx_t - \bmx^\star}^2_2 - \frac{\norm{\bmx_{T + 1} - \bmx^\star}^2_2}{T^{c + \varepsilon}} \\
        &\quad\leq D^2 + D^2 \sum_{t=2}^{T} \qty(\frac{1}{t^{c + \varepsilon}} - \frac{1}{(t - 1)^{c + \varepsilon}}) \\
        &\quad= D^2 + D^2\qty(\frac{1}{T^{c + \varepsilon}} - 1) \leq D^2,
    \end{align*}
    where the first inequality follows from \Cref{ass:X}.
\end{proof}

\section{Proof of \Cref{thm:strongly_convex}}\label{appendix:proof_regret}

\begin{proof}
    Similar to the argument of \Cref{lemma:self_bounding_property}, for any strongly convex function $f_t$ with modulus $\sigma_t > 0$ and for any optimal solution $\bmx^\star\in\cX$ to the offline constrained OCO as Eq.~\eqref{opt:COCO}, we have
    \begin{align}
        f_t(\bmx_t) - f_t(\bmx^\star) + \lambda_t\widehat{g}^+_t (\bmx_{t+1}) \leq \frac{F_t^2}{4\alpha_t} - \frac{\sigma_t}{2}\norm{\bmx^\star - \bmx_t}^2_2 + \frac{\alpha_t}{2} \norm{\bmx^\star - \bmx_t}^2_2 - \frac{\alpha_t}{2} \norm{\bmx^\star - \bmx_{t+1}}^2_2. \label{eq:self-bounding-for-strong}
    \end{align}
    Applying the above inequality \eqref{eq:self-bounding-for-strong} for the function $h_t$ defined by Eq.~\eqref{eq:def_h_t}, we have
    \begin{align}\label{eq:self_bounding_for_strong}
        h_t(\bmx_t) - h_t(\bmx^\star) + \lambda_t\widehat{g}^+_t(\bmx_{t+1}) \leq \frac{9F_t^2d^2}{4\alpha_t} - \frac{\sigma_t}{2}\norm{\bmx^\star - \bmx_t}^2_2 + \frac{\alpha_t}{2} \norm{\bmx^\star - \bmx_t}^2_2 - \frac{\alpha_t}{2} \norm{\bmx^\star - \bmx_{t+1}}^2_2.
    \end{align}
    Note that the function $h_t$ is also strongly convex with modulus $\sigma_t$ under \Cref{ass:ft_strong}.
    Since $\lambda_t\widehat{g}^+_t(\bmx_{t+1})$ is nonnegative, from Eq.~\eqref{eq:self_bounding_for_strong}, we have
    \begin{align*}
        h_t(\bmx_t) - h_t(\bmx^\star) \leq \frac{9F_t^2d^2}{4\alpha_t} - \frac{\sigma_t}{2}\norm{\bmx^\star - \bmx_t}^2_2 + \frac{\alpha_t}{2} \norm{\bmx^\star - \bmx_t}^2_2 - \frac{\alpha_t}{2} \norm{\bmx^\star - \bmx_{t+1}}^2_2.
    \end{align*}
    By taking summation over $t=1,2,\dots,T$, we have
    \begin{align*}
        \sum_{t=1}^{T} [h_t(\bmx_t) - h_t(\bmx^\star)] 
        &\leq \sum_{t=1}^{T} \frac{9F_t^2d^2}{4\alpha_t} + \sum_{t=1}^{T} \qty(\frac{\alpha_t}{2} - \frac{\alpha_{t-1}}{2} - \frac{\sigma_t}{2}) \norm{\bmx_t - \bmx}^2_2 \\
        &\leq \frac{9F^2d^2}{4} \sum_{t=1}^{T} \frac{1}{\alpha_t} + D^2 \sum_{t=1}^{T} \qty(\frac{\alpha_t}{2} - \frac{\alpha_{t-1}}{2} - \frac{\sigma}{2}),
    \end{align*}
    where the second inequality holds from \Cref{ass:X}.
    Plugging in $\alpha_t=\sigma t$, we have
    \begin{align*}
        \sum_{t=1}^{T} [h_t(\bmx_t) - h_t(\bmx^\star)] \leq \frac{9F^2d^2}{4} \sum_{t=1}^{T} \frac{1}{\sigma t}
        \leq \frac{9F^2d^2}{4\sigma} \qty(1 + \int_1^T \frac{1}{t} \dd{t}) = \frac{9F^2d^2}{4\sigma} \qty(1 + \log T).
    \end{align*}
    Similar to the proof of the convex case, since we have $\ex[t]{h_t(\bmx)} = \widehat{f}_t(\bmx)$ for any $\bmx\in\cX$ and from the inequality \eqref{eq:sup_f}, we have
    \begin{align*}
        \sum_{t=1}^{T} [f_t(\bmx_t) - f_t(\bmx^\star)] &\leq \sum_{t=1}^{T} [\widehat{f}_t(\bmx_t) - \widehat{f}_t(\bmx^\star)] + \sum_{t=1}^{T} 2\delta F_t\\
        &\leq \frac{9F^2d^2}{4\sigma} \qty(1 + \log T) + 2F \\
        &\leq \qty(\frac{9F^2d^2}{4\sigma} + 2F) \qty(1 + \log T),
    \end{align*}
    where the second inequality holds from \Cref{ass:ft_strong} the third inequality follows by letting $\delta = \frac{1}{T}$.

    Next, we show the cumulative hard constraint violation bound for fixed constraints.
    From Eq.~\eqref{eq:self_bounding_for_strong}, we have
    \begin{align*}
        \lambda_t \widehat{g}^+_{t} (\bmx_{t+1}) \leq \frac{9F^2_t d^2}{4\alpha_t} + \abs{h_t(\bmx_t) - h_t(\bmx^\star)} - \frac{\sigma_t}{2} \norm{\bmx^\star - \bmx_t}^2_2 + \frac{\alpha_t}{2} \norm{\bmx^\star - \bmx_t}^2_2 - \frac{\alpha_t}{2} \norm{\bmx^\star - \bmx_{t+1}}^2_2.
    \end{align*}
    By the definition of $\widehat{g}_{t}(\bmx)=\gamma_t[g_t(\bmx)]_+$, we have
    \begin{align*}
        [g_{t} (\bmx_{t+1})]_+ &\leq \frac{9F^2_t d^2}{4\alpha_t \lambda_t \gamma_t} + \frac{\abs{h_t(\bmx_t) - h_t(\bmx^\star)}}{\lambda_t \gamma_t} - \frac{\sigma_t}{2\lambda_t \gamma_t} \qty(\norm{\bmx^\star - \bmx_t}^2_2 +  \norm{\bmx^\star - \bmx_t}^2_2 - \norm{\bmx^\star - \bmx_{t+1}}^2_2) 
    \end{align*}
    By taking summation over $t=1,2,\dots,T$, and plugging $\alpha_t=\sigma t,\,\gamma_t=t^c$ into the above inequality, we have the following result:
    \begin{align*}
        \sum_{t=1}^{T} [g_{t} (\bmx_{t+1})]_+ &\leq \frac{9F^2d^2}{4\sigma}\sum_{t=1}^{T} \frac{1}{t^{3c + \varepsilon}} + \sum_{t=1}^{T} \frac{\abs{h_t(\bmx_t) - h_t(\bmx^\star)}}{t^{2c + \varepsilon}} \\
        &\leq \frac{27F^2d^2}{4\sigma} + \frac{3FdD(1+\varepsilon)}{\varepsilon} = O(d^2),
    \end{align*}
    where the second inequality follows from \Cref{lemma:useful_inequality}.
\end{proof}

\section{Proof of \Cref{thm:strongly_constraints}}\label{appendix:proof_constr}

\begin{proof}
    Similar to the proof of \Cref{thm:adversarial_constraint}, we can show the upper bound of $\sum_{t=1}^{T}\norm{\bmx_t - \bmx_{t+1}}^2_2$ as 
    \begin{align*}
        \sum_{t=1}^{T} \norm{\bmx_t - \bmx_{t+1}}^2_2 &\leq \sum_{t=1}^{T} \frac{2\lip (h_t) D}{\frac{1}{2}\alpha_t} + D^2 \leq \frac{12FdD}{\sigma} \qty(1 + \log T) + D^2,
    \end{align*}
    where the second inequality holds by plugging in $\alpha_t = \sigma t$.
    Since we have $[g_t(\bmx_t)]_+ - [g_t(\bmx_{t+1})]_+ \leq \frac{G^2}{4\beta} + \beta\norm{\bmx_t - \bmx_{t+1}}^2_2$ for any $\beta > 0$, by letting $\beta = \sqrt{\frac{T}{1 + \log T}}$, we have
    \begin{align*}
        \sum_{t=1}^{T} \qty([g_t(\bmx_t)]_+ - [g_t(\bmx_{t+1})]_+) 
        &\leq \frac{G^2T}{4\beta} + \beta \sum_{t=1}^{T} \norm{\bmx_{t} - \bmx_{t+1}}^2_2 \\
        &\leq \frac{G^2T}{4\beta} + \beta\qty(\frac{12FdD}{\sigma} \qty(1 + \log T) + D^2) \\
        &\leq \qty(\frac{G^2}{4} + \frac{12FdD}{\sigma}) \sqrt{T(1 + \log T)} + D^2 \sqrt{\frac{T}{1 + \log T}}.
    \end{align*}
    Finally, by combining the result of \Cref{thm:strongly_convex}, we have the following result:
    \begin{align*}
        \sum_{t=1}^{T} [g_t(\bmx_t)]_+ &\leq \sum_{t=1}^{T} [g_t(\bmx_{t+1})]_+ +  \qty(\frac{G^2}{4} + \frac{12FdD}{\sigma}) \sqrt{T(1 + \log T)} + D^2 \sqrt{\frac{T}{1 + \log T}} \\
        &\leq \frac{27F^2d^2}{4\sigma} + \frac{3FdD(1+\varepsilon)}{\varepsilon} + \qty(\frac{G^2}{4} + \frac{12FdD}{\sigma}) \sqrt{T(1 + \log T)} + D^2 \sqrt{\frac{T}{1 + \log T}} \\
        &\leq \qty(\frac{27F^2d^2}{4\sigma} + \frac{G^2}{4} + 3FdD \qty(1 + \frac{1}{\varepsilon} + \frac{4}{\sigma}) + D^2) \sqrt{T(1 + \log T)}.
    \end{align*}
\end{proof}

\end{document}